\documentclass{esann}
\usepackage[dvips]{graphicx}
\usepackage[latin1]{inputenc}
\usepackage{amssymb,amsmath,array}

\usepackage{colortbl}
\usepackage{xcolor}
\usepackage{multicol}
\usepackage{multirow}
\usepackage{tikz}
\usepackage{algorithm}
\usepackage{float}
\usepackage{bm}
\usepackage{verbatim}
\usepackage{dsfont}
\usepackage{algorithm,algpseudocode,float}
\usepackage{graphicx}
\usepackage{setspace}
\usepackage{comment}
\usepackage{hyperref}
\usepackage{arydshln}
\usepackage{amsthm}

\newtheorem{corollary}{Corollary}

\newtheorem{lemma}{Lemma}
\newtheorem{proposition}{Proposition}
\newtheorem{remark}{Remark}
\newtheorem{definition}{Definition}

\DeclareMathOperator*{\argmax}{arg \, max}

\newcommand\given[1][]{\mid}

\begin{document}

\title{Estimating Individual Treatment Effects through Causal Populations Identification}

\author{C\'eline Beji$^1$, Eric Benhamou$^{1,2}$, Micha\"el Bon$^3$, Florian Yger$^1$, Jamal Atif$^1$
%
%
\vspace{.3cm}\\
%
1- Paris-Dauphine University - PSL, LAMSADE, CNRS, MILES \\
Place du Mar\'echal de Lattre de Tassigny, 75016 Paris - FRANCE
%
\vspace{.05cm}\\
2- Ai Square Connect \hspace{3cm}
3- AdWay, Groupe Square\\
}

\maketitle
\footnote{ESANN 2020 proceedings, European Symposium on Artificial Neural Networks, Computational Intelligence
and Machine Learning. Bruges (Belgium), 22-24 April 2020.}

\begin{abstract}

Estimating the Individual Treatment Effect from observational data, defined as the difference between outcomes with and without treatment or intervention, while observing just one of both, is a challenging problems in causal learning.
In this paper, we formulate this problem as an inference from hidden variables and enforce causal constraints based on a model of four exclusive causal populations. We propose a new version of the EM algorithm, coined as Expected-Causality-Maximization (ECM) algorithm and provide hints on its convergence under mild conditions. We compare our algorithm to baseline methods on synthetic and real-world data and discuss its performances. 


\end{abstract}

\section{Introduction}
Estimating \emph{Individual Treatment Effect} (ITE) from observational data is central in many application domains. For instance in healthcare, where the treatment is a proper medical treatment and the desired effect is the recovery of the patient. Being able to target accurately and demonstrably the population responding to a treatment has strong beneficial consequences in terms of public health by boosting personalized medicine. That would indeed allow a precise distribution of drugs to the profile of patients they can address, whereas at present, such drug may be prohibited because it does not demonstrate a positive effect at the whole population level through the lens of current standard testings.
In the vein of the Rubin's causality framework~\cite{Rubin_1974,Imbens1997}, we cast this counterfactual learning problem as an inference problem with incomplete data. We consider $X$ the $\mathcal{X}$-valued random variable ($\mathcal{X}\subseteq {\mathbb{R}}^d$) representing the features of an individual and $T$ the treatment assignment binary indicator stating whether the treatment was assigned ($T=1$) or not ($T=0$).
We denote $Y_i(1)$ the binary outcome that would be observed if we assigned the treatment to individual $i$ and $Y_i(0)$ the one that would be observed if we did not (e.g. $Y_i(1) = 1$ meaning that an effect was observed after treating individual $i$).
ITE of individual $X_i=x$ is defined as the conditional mean difference in potential outcomes,
$\tau(x)=\mathbb{E}[Y_i(1)-Y_i(0)\given X_i=x]$.
The fundamental problem is that for any individual $i$, we only observe the \emph{factual outcome} $Y_i(t)$ corresponding to the outcome of the assignment, whereas the \emph{counterfactual} $Y_i(1-t)$ remains unknown~\cite{Imbens1997}.
As summarized in Table~\ref{tab:Behaviors-probabilities}, from each couple $Y_i=\{Y_i(0),Y_i(1)\}$, called the \emph{potential outcome}, we can define four mutually exclusive categories of response to the treatment~\cite{Wasserman2004}:
responders who display a positive outcome only when treated, anti-responders who display a positive outcome only when they are \emph{not} treated, doomed and survivors, who respectively never and always display a positive outcome. 

 \begin{table*}[!ht]
  \centering
\resizebox {12cm} {!} {
\begin{tabular}{ c | c  |c  |c }
  Responder (R) & Doomed (D)  & Survivor (S) & Anti-responder (A) \\
 \hline 
 $\{Y(1)=1,Y(0)=0\}$ & $\{Y(1)=0,Y(0)=0\}$ & $\{Y(1)=1,Y(0)=1\}$ & $\{Y(1)=0,Y(0)=1\}$ \\
\end{tabular}
    }
\caption{Potential outcome of each causal population 
}
 \label{tab:Behaviors-probabilities}
\end{table*}

Based on this typology of behaviors, we write the counterfactual learning problem as a parametric estimation of latent variables constrained by the causal groups properties. In this holistic approach, not only do we efficiently evaluate the ITE, but we are able to identify under mild assumptions the causality classes.

\section{Related Work}

\label{section:related work}
The literature on causal inference is abundant, and it is beyond the scope of this paper to cover it exhaustively, although \cite{Pearl_2009} is a good reference for a broad overview of this topic. Within Rubin's framework, baseline approaches consist in using treatment as a feature, or in learning two independent classifiers on the control and on the test datasets. This latter approach has the advantage of simplicity and versatility, but may lead to a selection bias. To overcome this problem, more sophisticated methods have been proposed. A first group on method consists in adaptations of classical machine learning methods. Examples are: (i) an SVM-like approach \cite{zaniewicz2013support} where two hyperplanes are introduced and properly optimized to separate class behaviors.
    (ii) a parametric Bayesian method \cite{alaa2017bayesian} for learning the treatment effects using a multi-task Gaussian process.
    (iii) several random forest-based approaches with  split criteria specifically adapted to the problem~\cite{wager2018estimation}.
Deep learning methods have also been put to good use with examples such as:
(i) a deep neural network architecture \cite{shalit2017estimating} able to learn classifiers on the test and control populations while enforcing the minimization of an integral probability metric between the distributions of these classifiers. This work builds upon \cite{johansson2016learning} where counterfactual inference has been tackled from the perspective of domain adaptation and representation learning. 
   (ii) A number of methods using neural networks have also recently emerged~\cite{yoon2018ganite,louizos2017causal}.
Interestingly, when the test and control populations have the same size, it is shown in~\cite{Jaskowski2012} that a variable change could be used, leading to the estimation of a unique probability distribution (allowing the straightforward use of classical methods).

Our approach is different from the ones above. In our case, with binary treatment and outcome, the population has a clear causal structure. We use this fact and model the population as a mixture of four causal groups. Then, from the general knowledge of the causal structure obtained by our method, we can derive the ITE and thus compare our results with ITE-specific methods.

\section{A Parametric Model for Causal Populations}
\label{sec:paramModel}
We model the whole population as a mixture of mutually exclusive causal groups coined as responders (R), doomed (D), survivors (S) and anti-responders (A). We denote their respective distributions as 
$\{f_k(. | \theta_k) \}_{k \in \{R,D,S,A\}}$, and $\pi_k$ their mixing probability.

For an individual, the specific group to which he belongs is determined by his outcome with and without treatment. Since we can never observe both simultaneously, we introduce $Z_i=\{z_{ik}\}_{k \in \{R,D,S,A\}}$ the discrete latent variable that represents the class probability of individual $i$. 


Our model implies several constraints on the distribution of this latent variable, obviously excluding two causal populations according to the factual outcome and the assigned treatment. For example, it appears from Table~\ref{tab:Behaviors-probabilities} that an individual $i$ with $Y_i(0)=0$ cannot be a survivor or an anti-responder. The probability distribution of these two classes can then be set to zero ($z_{iS}=z_{iA}=0$). Similar constraints can be applied for every value of the factual outcomes and are summarized in Table~\ref{tab:causality_constraints}. Our goal is to estimate the latent distribution, from which we can in particular derive the ITE.

 \begin{table*}[!ht]
  \centering
\resizebox {12cm} {!} {
    \begin{tabular}{p{2cm}||c|c|c|c}
  &$Y_i(0)=0$ & $Y_i(0)=1$ & $Y_i(1)=0$  & $Y_i(1)=1$ \\
 \hline
  Causality constraints &
 $ \left\{
    \begin{aligned}
        z_{iS} = z_{iA} = 0 \\
        z_{iR} + z_{iD} = 1
    \end{aligned}
\right. $
& $ \left\{
    \begin{aligned}
        z_{iR} = z_{iD} = 0 \\
        z_{iS} + z_{iA} = 1
    \end{aligned}
\right. $
& $ \left\{
    \begin{aligned}
        z_{iR} = z_{iS} = 0 \\
        z_{iD} + z_{iA} = 1
    \end{aligned}
\right. $
& $ \left\{
    \begin{aligned}
        z_{iD} = z_{iA} = 0 \\
        z_{iR} + z_{iS} = 1
    \end{aligned}
\right. $
\\
\end{tabular}
    }
\caption{The causality constraints ($C^*$)}
 \label{tab:causality_constraints}
\end{table*}




\begin{proposition}\label{prop:ITE}
Knowing the latent distribution 
, ITE writes as
\begin{equation}
    \label{eq:ITE}
    \tau(x) = (l_R(x)+l_S(x)) \mathbb{E}[\mathds{1}_{Y_i(1)=1}]
	  - (l_S(x)+l_A(x))
	  \mathbb{E}[\mathds{1}_{Y_i(0)=1}]
\end{equation}
where $l_C(x)= \frac{\pi_C f_C(X_i=x|\theta_C)}{\sum_{G\in \{R,D,S,A\}} \pi_G f_G(X_i=x|\theta_G)}$, $C \in \{R,D,S,A\}$.
\end{proposition}
\begin{proof} 
$\mathbb{E}[Y_i(1)=1|X_i=x] = \frac{\mathbb{E}[X_i(1)=x|Y_i(1)=1]P(Y_i(1)=1)}{P(X_i=x)}$ \\ $ \textcolor{white}{.} \hspace{38pt}
= \frac{\mathbb{E}[X_i(1)=x|X_i \in \{R,S\}]P(Y_i(1)=1)}{P(X_i=x)}
= \frac{\sum_{C \in \{R,S\}} \pi_C f_C(X_i=x|\theta_C) P(Y_i(1)=1)}{\sum_{C \in \{R,D,S,A\}}\pi_C f_C(X_i=x|\theta_C)}$
\end{proof}

\section{ECM Algorithm}
\label{Causal_EM_theory}
Our learning problem amounts to estimate the mixing coefficients $\{\pi_k\}_{k \in \{R,D,S,A\}}$, the  distributions parameters $\theta=\{\theta_k\}_{k \in \{R,D,S,A\}}$ and the latent distribution $q(z)$. For that matter, we consider the Expectation-Maximization (EM) algorithm, originally introduced in~\cite{Dempster_1977}, which is known to be an appropriate optimization algorithm for estimating the data distribution of hidden variables. 
We provide the EM algorithm with extra information about the possible groups for every observation (in spirit similarly to~\cite{ambroise2000algorithm} where a concept of authorized label set is used or to~\cite{come2009learning} which uses partial information). However, contrary to~\cite{ambroise2000algorithm, come2009learning}, we enrich this extra information with causal constraints derived from the structure of the problem.

\begin{algorithm}[!ht]
\caption{Expectation-Causality-Maximisation}\label{algo:causal_em}
\begin{algorithmic} 
\State \textbf{Initialisation:} initialise $q_0$ and compute $\pi_0$ and $\theta_0$ (M-step).
\State \textbf{While}(Not Converged) do
	\State \hspace{0.5 cm} Expected step:
	$q_{t+1} = \argmax_q(\mathcal{L}(q, \theta_t,\pi_t))$
	\State \hspace{0.5 cm} Causality step:
	Constraints on $q_{t+1}$ with $C^*$ (Table~\ref{tab:causality_constraints})
	\State \hspace{0.5 cm} Maximization step:
	$(\theta_{t+1},\pi_{t+1}) = \argmax_{\theta,\pi}(\mathcal{L}(q_{t+1},\theta,\pi))$
\State \textbf{End While}
\end{algorithmic}
\end{algorithm}

In Algorithm~\ref{algo:causal_em}, the Expectation step estimates the latent variables, the Causality step projects the solution on the causality constraints\footnote{Forcing to zero the probabilities $z_{ik}$ of the  two impossible groups and normalizing the sum.}
displayed in Table~\ref{tab:causality_constraints}, while the Maximization step maximises the likelihood $\mathcal{L}(q,\theta,\pi)$ as if the latent variables were not hidden. For a faster convergence, we initialize $q_0$ with a probability of half on each of the two remaining causal populations.
Our algorithm converges as by construction it necessarily increases the log-likelihood at each iterations and remains bounded by an evidence lower bound similar to EM given by
$\sum_{z|q(z|x,\theta,\pi) \neq 0} q(z|x,\theta,\pi) \log \frac{p(x,z|\theta,\pi)}{q(z|x,\theta,\pi)}$ .
Note that without information on the input features X, the distribution $q(z)$ is uniformly distributed between the two authorized groups.
In addition, the causality constraints enforce that two population labels are ruled out as they are not admissible. Under some specific assumptions, we can do even better and recover the true label (cf Proposition~\ref{CausalEM}). Thanks to this true label, the maximum likelihood problem is cast into four decoupled single-density maximum likelihood problems.  Under concavity of the likelihood for every distribution in the mixture, the log-likelihood converge not only locally but to a unique global maximum (corollary~\ref{CausalEM2}). We can summarize these findings by saying that the unsupervised learning problem, implied by our model of mixture, is transformed into a semi-supervised problem.

\begin{proposition}\label{CausalEM} 
Under causality constraints which excludes two populations (depending on treatment and observed outcome), and assuming the feature distribution conditionally to the group is the same  and independent of the treatment (i.e. $p(x|t, y) = p(x)$), each group will be identified to a unique causal population.

\end{proposition}
\begin{proof} (sketch)
    Under only the information of the causal constraints, $q(z|y,t)$ is distributed with the probability $\frac{1}{2}$ between the two possible populations. If we note $p(.)=p(.|\theta,\pi)$, then 
    $q(z|x,y,t)=\frac{p(x|z,y,t)q(z|y,t)}{p(x|y,t)} \propto \frac{p(x|z)}{p(x)}$ under causal constraints and assumption of uniform features.
\end{proof}

\begin{corollary}
\label{CausalEM2} Under causality constraints, if the log-likelihood of the distribution for a single mixture is concave, the ECM algorithm reaches the global optimum.
\end{corollary}
\begin{proof} (sketch)
    Preserving the properties of the EM algorithm, ECM converges to a local maximum. As a result of the identifiability (Prop.~\ref{CausalEM}) and the concavity for a single mixture, the local optimum is also global.
\end{proof}

\section{Experiments}

Gaussian distributions (for which parameters are the mean and variance denoted by $\theta=(\mu, \Sigma)$) are a natural choice for a mixture model. Once the model is learned, interpreting the $\mu_k$ as average elements of each causal population could be of great interest and could help to answer questions like "what does an average anti-responder looks like?" to improve treatment policies.
Because of the fundamental impossibility to access the true counterfactual label of any given observation, the true ITE cannot be known and thus it is unclear how to best assess the relevance of any model in real-world conditions. Hence, we need to test our model on synthetic and semi-synthetic datasets. 

We first designed a synthetic dataset with two covariates distributed as a mixture of four overlapping Gaussian distributions. We use two metrics standard for causal problems: the $\epsilon_{PEHE}$~\cite{hill2011bayesian} the AUUC (Area Under the Uplift Curve)~\cite{diemert2018large}. We compute these metrics for the optimal model (since it does not necessarily score maximally according to these metrics) and compare its values to the ones of the models we test.  
We also use the IHDP semi-synthetic dataset, compiled for causal effect estimation in~\cite{hill2011bayesian}. Here the underlying distribution of each causal population remains unknown, but the outcomes with and without treatment are both available. In that case, we use our model to predict the most likely counterfactual of any individual.

The results are reported out of a sample over $20$ trials and a Wilcoxon signed-rank test (with a confidence level of $5\%$) is used to confirm the significance of the results.
We compare our method to standard baselines that provide competitive results with respect to the state of the art~\cite{alaa2018limits}: the approach using two separate classification models (LR2), the approach using the treatment variable as feature (LR1) and the model based on the class variable transformation~\cite{Jaskowski2012} (LRZ), each using logistic regressions as classifiers.

\begin{table}[!ht]
    \begin{center}
        \begin{tabular}{rllll}
            \hline 
            & \multicolumn{2}{c}{Synthetic dataset} & \multicolumn{2}{c}{IHDP}\\
               & $\epsilon_{PEHE}$ & AUUC
             & $\epsilon_{PEHE}$ & AUUC\\
            \hline
                  Ref.
                & 0.24 
                & 1488 
                & .
                & 3149 
                \\ \hdashline
                 LR1
                & 0.57 +/- 0.08 
                & 742 +/- 175
                & 0.66 +/- 0.08
                & 2202 +/- 625\\
                 LR2
                & 0.79 +/- 0.08
                & 943 +/- 206
                & 0.67 +/- 0.07
                & 2168 +/- 618\\
                 LRZ
                & .
                & 939 +/- 208
                & .
                & 2191 +/- 558 \\
                 ECM
                & \textbf{0.27 +/- 0.04}
                & \textbf{1512 +/- 203}
                & \textbf{0.59 +/- 0.09}
                & 2226 +/- 580\\
            \hline 
        \end{tabular}
    \end{center}
    \label{tab:results}
    \caption{Experimental results on synthetic and real datasets.}
    
\end{table}

\section{Results and Conclusion}
Compared to the baselines, our results are clearly the ones closest to optimality on both synthetic and real datasets. 
Moreover, our model is intrinsically more interpretable than the compared baselines as the parameters of the distributions of the causal groups provide information about the causal mechanism at play. 
Finally, our model is versatile and can be adapted to multiple treatments~~\cite{frolich2004programme}, non-compliance to treatment cases~\cite{Imbens1997} or separate labels~\cite{yamane2018uplift}.


\begin{footnotesize}


\bibliographystyle{unsrt}
\bibliography{biblio}

\end{footnotesize}


\end{document}